%% file: main.tex
\documentclass[nohyperref]{article}

\usepackage{microtype}
\usepackage{graphicx}
\usepackage{subfigure}
\usepackage{booktabs} % for professional tables

\usepackage{hyperref}

\usepackage[accepted]{icml2022_pods}

\usepackage[utf8]{inputenc} % allow utf-8 input
\usepackage[T1]{fontenc}    % use 8-bit T1 fonts
\usepackage{hyperref}       % hyperlinks
\usepackage{url}            % simple URL typesetting
\usepackage{booktabs}       % professional-quality tables
\usepackage{amsfonts}       % blackboard math symbols
\usepackage{nicefrac}       % compact symbols for 1/2, etc.
\usepackage{microtype}      % microtypography
\usepackage{xcolor}         % colors
\usepackage[fleqn]{amsmath}
\usepackage{bm}
\usepackage{enumitem}
\usepackage{amsfonts}
\usepackage{amsmath}
\usepackage{amssymb}
\usepackage{mathtools}
\usepackage{amsthm}
\usepackage{graphicx}
\usepackage{xcolor}
\usepackage{color}

\usepackage[capitalize,noabbrev]{cleveref}

\theoremstyle{plain}
\newtheorem{theorem}{Theorem}[section]

\newtheorem{lemma}[theorem]{Lemma}
\newtheorem{corollary}[theorem]{Corollary}
\theoremstyle{definition}
\newtheorem{definition}[theorem]{Definition}

\theoremstyle{remark}

\usepackage[textsize=tiny]{todonotes}

\icmltitlerunning{Data Augmentation vs. Equivariant Networks: A Theory of Generalization on Dynamics Forecasting}

\begin{document}

\twocolumn[
\icmltitle{Data Augmentation vs. Equivariant Networks: \\A Theory of Generalization on Dynamics Forecasting}
% An Equivariance-based Theory Framework for Dynamics Forecasting}

% It is OKAY to include author information, even for blind
% submissions: the style file will automatically remove it for you
% unless you've provided the [accepted] option to the icml2022
% package.

% List of affiliations: The first argument should be a (short)
% identifier you will use later to specify author affiliations
% Academic affiliations should list Department, University, City, Region, Country
% Industry affiliations should list Company, City, Region, Country

% You can specify symbols, otherwise they are numbered in order.
% Ideally, you should not use this facility. Affiliations will be numbered
% in order of appearance and this is the preferred way.

\begin{icmlauthorlist}
\icmlauthor{Rui Wang}{sd}
\icmlauthor{Robin Walters}{neu}
\icmlauthor{Rose Yu}{sd}
%\icmlauthor{}{sch}
\end{icmlauthorlist}

\icmlaffiliation{sd}{University of California San Diego}
\icmlaffiliation{neu}{Northeastern University}

\icmlcorrespondingauthor{Rui Wang}{ruw020@ucsd.edu}

% You may provide any keywords that you
% find helpful for describing your paper; these are used to populate
% the "keywords" metadata in the PDF but will not be shown in the document
\icmlkeywords{Machine Learning, ICML}

\vskip 0.3in
]

% this must go after the closing bracket ] following \twocolumn[ ...

% This command actually creates the footnote in the first column
% listing the affiliations and the copyright notice.
% The command takes one argument, which is text to display at the start of the footnote.
% The \icmlEqualContribution command is standard text for equal contribution.
% Remove it (just {}) if you do not need this facility.

\printAffiliationsAndNotice{}  % leave blank if no need to mention equal contribution
%\printAffiliationsAndNotice{\icmlEqualContribution} % otherwise use the standard text.

\begin{abstract}
Exploiting symmetry in dynamical systems is a powerful way to improve the generalization of deep learning. The model learns to be invariant to transformation, and hence is more robust to distribution shift.  Data augmentation and equivariant networks are two major approaches to inject symmetry into learning. However, their exact role in improving generalization is not well understood.  In this work, we derive the generalization bounds for data augmentation and equivariant networks, characterizing their effect on learning in a unified framework.
Unlike most prior theories for the i.i.d. setting, we focus on non-stationary dynamics forecasting with complex temporal dependencies. 

% Exploiting symmetry in structured data is a powerful way to improve the learning and generalization ability of deep learning models. Data augmentation and equivariant neural nets are two of the main approaches for enabling neural nets to preserve symmetries. Since real-world data is rarely strictly symmetric, recently, several approximately equivariant networks have also been introduced. However, most prior works on symmetry are either purely empirical or theoretical study under i.i.d setting. Thus in this work, we take the first step to theoretically study the generalizability of data augmentation, equivariant networks and approximately equivariant networks on the task of non-stationary dynamics forecasting. 
\end{abstract}

\section{Introduction}
\input{sections/intro}

% \section{Related Work}
% \input{sections/relate}

\section{Background}
\input{sections/background}

\section{Data Augmentation vs. Equivariant Nets}
% \subsection{Problem Definition}
% \input{sections/general_framework/prob_setup}

% \subsection{Data Augmentation vs. Equivariance vs. Approximate Equivariance}
% \input{sections/general_framework/sym_intro}
% \subsection{Data Aug vs. Equivariant Network}
% \input{sections/iid}

% \section{From i.i.d to Dynamic Forecasting}
% \subsection{Statistical Sequence Learning}
\input{sections/general_framework/stats_intro}

\subsection{Generalization Bound for  Dynamics Forecasting}
\input{sections/dynamics_forecasting}

\subsection{Effect of Symmetry}

\input{sections/bound_sym}

\section{Conclusion}
\input{sections/conclusion}
% \section*{References}

%%%%%%%%%%%%%%%%%%%%%%%%%%%%%%%%%%%%%%%%%%%%%%%%%%%%%%%%%%%%
% \section*{Checklist}
% \input{sections/checklist}

%%%%%%%%%%%%%%%%%%%%%%%%%%%%%%%%%%%%%%%%%%%

\bibliography{example_paper}
\bibliographystyle{icml2022}

%%%%%%%%%%%%%%%%%%%%%%%%%%%%%%%%%%%%%%%%%%%%%%%%%%%%%%%%%%%%%%%%%%%%%%%%%%%%%%%
%%%%%%%%%%%%%%%%%%%%%%%%%%%%%%%%%%%%%%%%%%%%%%%%%%%%%%%%%%%%%%%%%%%%%%%%%%%%%%%
% APPENDIX
%%%%%%%%%%%%%%%%%%%%%%%%%%%%%%%%%%%%%%%%%%%%%%%%%%%%%%%%%%%%%%%%%%%%%%%%%%%%%%%
%%%%%%%%%%%%%%%%%%%%%%%%%%%%%%%%%%%%%%%%%%%%%%%%%%%%%%%%%%%%%%%%%%%%%%%%%%%%%%%
\newpage
\appendix
\onecolumn
\section{Appendix}
\input{sections/app}%%%%%%%%%%%%%%%%%%%%%%%%%%%%%%%%%%%%%%%%%%%%%%%%%%%%%%%%%%%%%%%%%%%%%%%%%%%%%%%
%%%%%%%%%%%%%%%%%%%%%%%%%%%%%%%%%%%%%%%%%%%%%%%%%%%%%%%%%%%%%%%%%%%%%%%%%%%%%%%

\end{document}

%% file: sections/intro.tex
Symmetry plays an important role in the success of deep learning; incorporating symmetries into layers or training deep neural nets can improve generalizability and robustness \cite{bronstein2021geometric, shorten2019survey, lopes2019improving, wang2021incorporating}. There are two main techniques to train models that preserve symmetries. In data augmentation, one add samples to the training set which are transformed versions of other samples.  This enables the model to learn invariance to symmetry transformations and noise \cite{hernandez2018data, dao2019kernel, hernandez2018further, rajput2019does, ratner2017learning, zhou2022do, perez2017effectiveness, wen2020time}. The other line of work is the design of equivariant neural networks, which have also achieved remarkable success in learning image data \cite{worrall2019deep, cohen2016steerable, weiler2019e2cnn, cohen2018spherical} and physical dynamics \cite{wang2021incorporating, wang2022approximately, shi2021learning}. 

Here, we consider the problem of learning  dynamical systems, where the data is non-i.i.d and the symmetry is rarely perfect. A perfectly equivariant model may have trouble learning partial or approximated symmetries in real-world data. Thus, some work has recently explored the idea of building approximately equivariant models and empirically demonstrated the benefits of it in modeling real-world data \cite{van2022relaxing, romero2021learning, finzi2021residual}. For example, \cite{wang2022approximately} designed approximately equivariant models by relaxing the weight sharing schemes in the equivariant convolution networks. 

Most prior works on data augmentation techniques and equivariant networks are  purely empirical. There is no theory that characterizes and compares their behavior. Even though both approaches exploit the symmetry in the prediction task, their exact role  to improve generalization is not well understood. Furthermore, it is not clear when these approaches are beneficial or in what regime one approach is preferred over the other. While there exist some theories of generalization for data augmentation, all of them are under the i.i.d. assumption \cite{rajput2019does, ratner2017learning, sannai2021improved}. For equivariant networks, the theory of generalization is even more scarce.

In this work, we present a theory of generalization for dynamics forecasting, where the data are non-stationary and non-mixing time series.  We theoretically analyze and compare the generalization strength of data augmentation versus equivariant networks. We show that when the underlying dynamics is symmetric, equivariant networks achieve a tighter generalization bound than data augmentation. Furthermore,   when the symmetries in the data are only approximate,  the generalization bound for approximately equivariant networks \cite{wang2022approximately} is further improved. 

In summary, our contributions include:
\vspace{-7pt}
\begin{itemize}
\item We formally characterize the behavior of dynamic forecasting with deep learning under the assumption that the underlying dynamical system preserves a certain amount of symmetry. 
\item We derive the generalization  bounds for data augmentation  and  equivariant networks, including both the strict and approximately equivariant networks for non-stationary and non-mixing time series. 
\item We prove that equivariant networks have a tighter generalization upper bound than data augmentation. When the data do not have perfect symmetries, approximately equivariant models tend to have better generalizability than the other two approaches. 
\end{itemize}
\vspace{-12pt}

%% file: sections/background.tex
\subsection{Statistical Learning Theory for Time Series.}

We consider forecasting deterministic dynamics where the learner receives $N$ observed time series $\{X^{(1)}, ..., X^{(N)}\}$ with length $T$ of a dynamical system \cite{Wang2021MetaLearningDF, wang2021bridging}. Each time series $X^{(i)}$ is a sample from a dynamical system where the system parameters are drawn i.i.d. from a given distribution. Even though the system parameters are independently sampled,   each time series can be highly non-stationary and exhibit complex  dependencies.

Denote $Z_t^{(i)} = (X_{t-k-1:t-1}^{(i)}, X_{t}^{(i)}) \in \mathcal{X}^k\times\mathcal{X}$ as a training sample (a subsequence  of time series $i$ at time $t$). $X_{t-k-1:t-1}^{(i)}$ and $X_{t}^{(i)}$ are the input and output of a forecasting model. For a loss function $\mathcal{L}: \mathcal{X}  \times \mathcal{X} \to [0, \infty)$ and a hypothesis set $\mathcal{F}$ of functions that map from $\mathcal{X}^k$ to $\mathcal{X}$, we want to minimize its empirical risk:
\[R_n(\theta) = \frac{1}{N} \sum_{i=1}^N\sum_{t=1}^T q_t \mathcal{L}(f_{\theta}(X_{t-k-1:t-1}^{(i)}), X_{t}^{(i)})\]
where $\theta$ represent the parameters in $f$. For simplicity, we use $L(\theta, Z_t^{(i)})$ to denote $\mathcal{L}(f_{\theta}(X_{t-k-1:t-1}^{(i)}), X_{t}^{(i)})$

Note that $q_1,. . , q_T$ are real numbers, which in the standard statistical learning scenarios are chosen to be all equal to $\frac{1}{T}$. We follow the time series forecasting setting in \cite{kuznetsov2020discrepancy}. For non-stationary dynamics, different $Z_t$ may follow different distributions, and thus distinct weights could be assigned to the errors made on different sample points, depending on their relevance to forecasting the future $Z_{T+1}$.
The learning objective is to find a $\theta$ that achieves a small test error, $\mathbb{E} L(\theta, Z_{T+1})$.

To derive the generalization bound, \citet{bousquet2003introduction} and \citet{rakhlin2014statistical} generalizes the classic Rademacher Complexity \cite{gnecco2008approximation} to time series learning, as defined below,

\begin{definition}[Sequential Rademacher Complexity,  \citet{bousquet2003introduction,rakhlin2014statistical}]
Given a function class $\mathcal{G} \subset \mathbb{R}^\mathcal{Z}$, we define the sequential Rademacher complexity of class $\mathcal{G}$ as:
$$\mathcal{R}_T^{sq}(\mathcal{G}) = \mathbb{E}_{\bm{z}}\mathbb{E}_{\bm{\sigma}} [\text{sup}_{g\in\mathcal{G}}\sum_{t=1}^T\sigma_t q_t g(\bm{z}_t(\bm{\sigma}))]$$
where $\bm{z}$ is a real-valued complete binary tree that is a sequence $(z_1, . . . , z_T)$ of $T$ mappings  $z_t : \{\pm 1\}^{t-1} \to \mathbb{R}$ for $t \in [1,...,T]$, and $\bm{\sigma}$ is a sequence of Rademacher random variables, which is also a path in the tree $\bm{\sigma} = (\sigma_1,...,\sigma_{T-1}) \in \{\pm 1\}^{T-1}$. 

In our forecasting setting, the sequential Rademacher complexity of a loss class can be defined more specifically as:
\[
\mathcal{R}_T^{sq}(L\circ\Theta)= \mathbb{E}_{\bm{\sigma}}\!\left[\text{sup}_{\theta\in\Theta}\frac{1}{N} \sum_{i=1}^N\sum_{t=1}^T\sigma_t q_t L(\theta, Z_t^{(i)}))\right]
\]
\end{definition}

\subsection{Equivariance and Invariance.} 
Symmetry is often described through the equivariance or invariance of a given function. In this subsection, we give the formal definitions of data augmentation, equivariance and approximate equivariance in the deterministic dynamics forecasting setting. 
\paragraph{Equivariant Functions} A function $f$ respecting the symmetry coming from a group $G$ is said to be equivariant.
\begin{definition}[G-equivariant function]
%Formally, let $f $ be a function.
Assume a group representation $\rho_{\text{in}}$ of $G$ acting on $X$ and $\rho_{\text{out}}$ acting on $Y$. We say a function $f\colon X \to Y$ is \textbf{$G$-equivariant} if 
$$
f( \rho_{\text{in}}(g)(x)) = \rho_{\text{out}}(g) f(x)
$$
for all $x \in X$ and $g \in G$. The function $f$ is \textbf{$G$-invariant} if $f( \rho_{\text{in}}(g)(x)) = f(x)$ for all $x \in X$ and $g \in G$.  This is a special case of equivariance for the case $\rho_{\mathrm{out}}(g) = 1$. Equivariant neural networks \cite{cohen2016steerable, weiler2019e2cnn, cohen2016group} learn equivariant functions through weight-sharing and weight-tying.
\end{definition}

We define equivariance error, which quantifies the amount of symmetry the function $f$ contains.
\begin{definition}[Equivariance Error]
Let $f \colon X \to Y$ be a function and $G$ be a group. Assume that $G$ acts on $X$ and $Y$ via representation $\rho_{\text{in}}$ and $\rho_{\text{out}}$. Then the \textbf{equivariance error} of $f$ is 
$$
\|f\|_{EE} = \sup_{x,g}
%\frac{
\|f(\rho_{\text{in}}(g)(x)) - \rho_{\text{out}}(g)f(x)\|.
%}{\|x\|\|g\|}
$$
\end{definition}
For strictly equivariant functions, we have $\epsilon = 0$. But for real-world dynamics, the symmetry is often approximately equivariant, defined below:
\begin{definition}[Approximate Equivariance]\label{def_app}
$f \colon X \to Y$ is $\epsilon$-approximately equivariant if and only if $\| f\|_\mathrm{EE} < \epsilon.$ 
% Let $f \colon X \to Y$ be a function and $G$ be a group. Assume that $G$ acts on $X$ and $Y$ via representations $\rho_{\text{in}}$ and $\rho_{\text{out}}$. We say $f$ is $\epsilon$-approximately $G$-equivariant if for any $g \in G$, 
% $$
%     \|f(\rho_{\text{in}}(g)(x)) - \rho_{\text{out}}(g)f(x)\| \leq \epsilon.
% $$
\end{definition}
Several recent work have designed approximately equivariant networks \cite{wang2022approximately, van2022relaxing, finzi2021residual} to learn the approximate functions. In this work, we assume the equivariance errors of trained approximately equivariant models is less or equal to the true data equivariance errors.

% , which is a valid assumption based on the empirical study shown in Figure 4 in \citet{wang2022approximately} 

\paragraph{Data Aug. Introduces Symmetry.}
Consider a finite group $G$ that acts on the observed time series, we assume that for any $g \in G$, there is a certain amount of symmetry in the distribution, that is $Z_t^{(i)} \approx_d g Z_t^{(i)}, Z_t^{(i)} \sim \mathbb{P}$. We assume the group transformations are norm-preserving, i.e. $\|g\| = 1 \; \forall g \in G$, such as rotation and translation.

\begin{definition}[Data Augmentation]
Given a finite group $G$, we assume the augmented samples are the original samples applied with transformations uniformly sampled from the group. In other words, for every sample $Z_t^{(i)}$ in the original training set, we add samples $\{gZ_t^{(i)}, g\in G\}$. Then the augmented training set is the $|G|$ times bigger than the original training set. 
\end{definition}

%% file: sections/general_framework/stats_intro.tex
We derive generalization bounds for data augmentation and  equivariant networks. We show that the strictly equivariant networks can outperform  data augmentation. When the underlying dynamics are approximately symmetric, approximately equivariant estimator can outperform both data augmentation estimator and strictly equivariant networks. 

\subsection{Population and Empirical Risk Minimizers}
We first define the population and the empirical risk minimizers for data augmentation, perfectly equivariant models and approximately equivariant models based on the dynamic forecasting setting defined in the previous section. 
\begin{itemize}[leftmargin=*]
    \item Population minimizer: $\theta^* = \text{argmin}_{\theta \in \Theta} \mathbb{E}[L(\theta, Z)]$
    \item Empirical minimizer: \\
    \vspace{-15pt}
        \begin{align}
        \hat{\theta}_n = \text{argmin}_{\theta \in \Theta}  \frac{1}{N} \sum_{i=1}^N\sum_{t=1}^T q_t  L(\theta, Z_t^{(i)})\nonumber
        \end{align}
    \item Empirical minimizer for data augmentation: 
    \vspace{-6pt}
    \begin{align}
        \hat{\theta}_G &=  \text{argmin}_{\theta \in \Theta} \frac{1}{N} \sum_{i=1}^N\sum_{t=1}^T q_t \mathbb{E}_G[L(\theta, gZ_t^{(i)})] \nonumber\\
        & = \text{argmin}_\theta \frac{1}{N} \sum_{i=1}^N\sum_{t=1}^T q_t \bar{L}(\theta, Z_t^{(i)})\nonumber
    \end{align}
    where $\bar{L}$ is the orbit-averaging loss because of data augmentation based on the definition in \cite{Chen2020AGF}.
    \item Empirical minimizer for perfectly equivariant models: 
    \vspace{-5pt}
    \begin{equation}
     \begin{aligned}
    &\hat{\theta}_E =\text{argmin}_{\theta \in \Theta_E } \frac{1}{N} \sum_{i=1}^N\sum_{t=1}^T q_t  L(\theta, Z_t^{(i)}),\nonumber \\ 
    \Theta_E& = \{\theta:f_{\theta}( \rho_{\text{in}}(g)(x)) = \rho_{\text{out}}(g) f_{\theta}(x), \forall g\in G\}\nonumber
       \end{aligned}
    \end{equation}
    \item Empirical minimizer for approximately equivariant nets: 
    \vspace{-5pt}
    \begin{equation}
     \begin{aligned}
     &\hat{\theta}_{AE} = \text{argmin}_\theta \frac{1}{N} \sum_{i=1}^N\sum_{t=1}^T q_t  L(\theta, Z_t^{(i)}),\nonumber  \\ 
     \Theta_{AE} &= \{\theta: \text{sup}_{g\in G}\|f_{\theta}( \rho_{\text{in}}(g)(x)) = \rho_{\text{out}}(g) f_{\theta}(x)\|_2 \leq \epsilon\}\nonumber
    \end{aligned}
    \end{equation}
    
\end{itemize}
where $\Theta$ is the parameter space without symmetry inductive biases imposed, $\Theta_{E}$ is the parameter space of all equivariant functions, and $\Theta_{AE}$ is the parameter space of all approximately equivariant functions.

% \begin{definition}[Sequential Covering Number]\cite{bousquet2003introduction,rakhlin2014statistical}
% A set $\mathcal{V}$ of $\mathbb{R}$-valued trees of depth $T$ is a sequential $\alpha$-cover (with respect to $q$-weighted $l_p$ norm) of a function class $\mathcal{G}$ on a tree $\bm{z}$ of depth $T$ if for all $g \in \mathcal{G}$ and all $\bm{\sigma} \in \{\pm 1\}^T$, there is $\bm{v} \in \mathcal{V}$ such that
% $$(\sum_{t=1}^T|\bm{v}_t(\bm{\sigma}) - g(\bm{z}(\bm{\sigma}))|^p)^{\frac{1}{p}} \leq \|\bm{q}\|_q^{-1}\alpha$$
% The (sequential) covering number $\mathcal{N}_p(\alpha, \mathcal{G}, \bm{z})$ of a function class $\mathcal{G}$ on a given tree $\bm{z}$ is defined to be the size of the minimal sequential cover. the expected covering number is defined as $\mathbb{E}_{\bm{z}}[\mathcal{N}_p(\alpha, \mathcal{G}, \bm{z})].$
% \end{definition}

% In our forecasting setting, it can be defined more specifically as below:
% $$\mathcal{R}_T^{sq}(L\circ\Theta)= \text{sup}_{\bm{x}}\text{sup}_{\bm{y}}\mathbb{E}_{\bm{\sigma}}[\text{sup}_{h\in\mathcal{H}}\sum_{t=1}^T\sigma_t q_t L(f(\bm{x}_t(\bm{\sigma}_{1:t-1})), \bm{y}_t(\bm{\sigma}_{1:t-1}))]$$
% Where $\bm{x}$ and $\bm{y}$ are $\mathcal{X}^k$ and $\mathcal{X}$ valued complete binary tree of depth $T$.

%% file: sections/dynamics_forecasting.tex
\cite{kuznetsov2020discrepancy} presented a data-dependent learning bound for the general scenario of non-stationary non-mixing stochastic processes. Yet, our focus is forecasting deterministic dynamics. Since the dynamics is non-stationary, we define a discrepancy measure  to characterize the distributional shift between the training and test sets: 

\begin{definition}[Discrepancy Measure]
We use $\text{disc}_T(\bm{q})$ to denote the discrepancy between target distribution and the distribution of the trainiig samples. 
    \begin{equation}
     \begin{aligned}
     \text{disc}_T(\bm{q}) = \text{sup}_{\theta \in \Theta}\left|\mathbb{E}\left[\sum_{t=1}^T q_t L(\theta, Z_t)\right] - \mathbb{E}L(\theta, Z_{T+1})\right|\nonumber
    \end{aligned}  
    \end{equation}
\end{definition}
We prove that the upper bound of the generalization error on dynamics forecasting is controlled by the sequential Rademacher complexity and the discrepancy measure of the temporal distributional shift. 
\begin{theorem}\label{thm:bounds-dyanmics}
For any $\delta > 0$, with probability at least $1-\delta$, the following inequality holds for all $\hat{\theta} \in \Theta$ and all $\alpha = \|\bm{q}\|_2/2>0$:
    \begin{equation}
     \begin{aligned}
     &\mathbb{E} \mathcal{L}(\hat{\theta}, Z_{T+1}) - \mathbb{E} \mathcal{L}(\theta^*, Z_{T+1}) \\
     &\leq \: 2\text{disc}_T(\bm{q}) + 6M\sqrt{4\pi \log T}\mathcal{R}_T^{sq}(L\circ\Theta) \\
     &+ \;\sqrt{\frac{2 \log (2/\sigma)}{N}} + \|\bm{q}\|_2(M \sqrt{8 \text{log}\frac{1}{\delta}} + 1) \nonumber
    \end{aligned}  
    \end{equation}
\end{theorem}
Full proof can be found in the Appendix \ref{proof:bounds-dyanmics}. Note that our result is consistent with the conclusion in \citet{kuznetsov2020discrepancy} for stochastic dynamics. 

%% file: sections/bound_sym.tex
We derive generalization bounds for forecasting nonstationary dynamics with data augmentation, perfectly equivariant networks, and approximately equivariant networks based on Theorem \ref{thm:bounds-dyanmics}. Following \citet{Chen2020AGF}, we use the Wassertein distance to measure the closeness of the original distribution to the distribution under group transformations. The full proofs of the following corollary \ref{thm:bounds-dyanmics-aug}, \ref{thm:bounds-dyanmics-sym}, \ref{thm:bounds-dyanmics-approx-sym} can be found in Appendix \ref{proof:bounds-dyanmics-sym}.

We generalize Theorem 3.4 in \citet{Chen2020AGF} from the i.i.d case to  non-stationary dynamics forecasting:
\begin{corollary}[Data Augmentation]\label{thm:bounds-dyanmics-aug}
Let $L(\theta, \cdot)$ be uniformly Lipschitz w.r.t. $\theta$ with Lipschitz constant $\|L\|_{\text{Lip}}$. For any $\delta > 0$, with probability at least $1-\delta$, the following inequality holds: 
\begin{equation}
\begin{aligned}
& \mathbb{E} \mathcal{L}(\hat{\theta}_G, Z_{T+1}) - \mathbb{E} \mathcal{L}(\theta^*, Z_{T+1}) \\
& \leq 2\text{disc}_T(\bm{q}) +  6M\sqrt{4\pi\log T}\mathcal{R}_T^{sq}(\bar{L}\circ\Theta) + \sigma \\
& + \textrm{max}_{t,i}\;\|L\|_{\text{Lip}}\cdot\mathbb{E}_G [\mathcal{W}(Z_{T+1}, g Z_{T+1}) + q_t\mathcal{W}(Z_{t}^{(i)}, g Z_{t}^{(i)})] \nonumber
\end{aligned}
\end{equation}
where $\sigma =  \sqrt{\frac{2 \log (2/\sigma)}{N}} +\|\bm{q}\|_2(M\sqrt{8\text{log}\frac{1}{\delta}} + 1)$. 
\end{corollary}

We can see that the performance gain of data augmentation is governed by a bias term $\textrm{max}_{t,i}\;\|L\|_{\text{Lip}}\cdot\mathbb{E}_G [\mathcal{W}(Z_{T+1}, g Z_{T+1}) + q_t\mathcal{W}(Z_{t}^{(i)}, g Z_{t}^{(i)})]$, which vanishes under exact symmetry and the sequential Rademacher complexity reduction because of the group orbit averaging over the loss function. 

The difference in sequential Rademacher complexity between the data augmentation estimator and regular estimator can further be bounded as
\begin{equation}
\begin{aligned}
&\mathcal{R}_T^{sq}(\bar{L}\circ\Theta) - \mathcal{R}_T^{sq}(L\circ\Theta)\\
& \leq \Delta + \textrm{max}_{t,i}\;\|L\|_{\text{Lip}}\cdot\mathbb{E}_G [q_t\mathcal{W}(Z_{t}^{(i)}, g Z_{t}^{(i)})]
\end{aligned}
\end{equation}
where $\Delta  = \mathbb{E}_{\bm{\sigma}}[\text{sup}_{\theta\in\Theta}\sum_{t=1}^T\sigma_t q_t \mathbb{E}_GL(\theta, gZ_t)]-\mathbb{E}_{\bm{\sigma}}\mathbb{E}_G[\text{sup}_{\theta\in\Theta}\sum_{t=1}^T\sigma_t q_t L(\theta, gZ_t)] \leq 0$.

Here $\Delta$ corresponds to the "variance reduction term" defined in \citet{Chen2020AGF}. When $\Delta$ is small, data augmentation has a strong effect on improving generalizability.

To compare the generalization bounds of data augmentation and equivariant networks, we first need to prove the following lemma. 
\begin{lemma}\label{lemma-complexity}
$\mathcal{R}_T^{sq}(L\circ\Theta_E)\leq \mathcal{R}_T^{sq}(\bar{L}\circ\Theta) $
\end{lemma}
The proof can be found at the end of Appendix \ref{proof:bounds-dyanmics-sym}.

Next, we derive the generalization bound for strictly equivariant networks. 
\begin{corollary}[Equivariant Networks]\label{thm:bounds-dyanmics-sym}
Let $L(\theta, \cdot)$ be uniformly Lipschitz w.r.t. $\theta$ with Lipschitz constant $\|L\|_{\text{Lip}}$. For any $\delta > 0$, with probability at least $1-\delta$, the following inequality holds:
{\setlength{\mathindent}{0cm}
\begin{equation}
\begin{aligned}
&\mathbb{E} \mathcal{L}(\hat{\theta}_E, Z_{T+1}) - \mathbb{E} \mathcal{L}(\theta^*, Z_{T+1}) \\ &\leq 2\text{disc}_T(\bm{q}) + 6M\sqrt{4\pi \log T}\mathcal{R}_T^{sq}(L\circ\Theta_E)  + \sigma\nonumber \\
& + \|L\|_{\text{Lip}}\cdot\mathbb{E}_G \mathcal{W}(Z_{T+1}, g Z_{T+1})
\end{aligned}
\end{equation}
}
where $\sigma =  \sqrt{\frac{2 \log (2/\sigma)}{N}} +\|\bm{q}\|_2(M\sqrt{8\text{log}\frac{1}{\delta}} + 1)$. 
\end{corollary}
From Lemma \ref{lemma-complexity}, we have $\mathcal{R}_T^{sq}(L\circ\Theta_E)\leq \mathcal{R}_T^{sq}(\bar{L}\circ\Theta) $. Hence, Corollary \ref{thm:bounds-dyanmics-sym} indicates that equivariant networks have a tighter generalization bound than data augmentation. In particular, the generalization bound of data augmentation in Corollary \ref{thm:bounds-dyanmics-aug} has an additional bias term $\textrm{max}_{t,i}\;\|L\|_{\text{Lip}}\cdot\mathbb{E}_G [q_t\mathcal{W}(Z_{t}^{(i)}, g Z_{t}^{(i)})]$. This term vanishes when the data are perfectly symmetric.

However, in real-world scenarios, the data are very rarely perfect symmetric. We further analyze the generalization behavior of a class of approximate equivariant models:
\begin{corollary}[Approximate Equivariance]\label{thm:bounds-dyanmics-approx-sym}
Let $L(\theta, \cdot)$ be uniformly Lipschitz w.r.t. $\theta$ with a Lipschitz constant $\|L\|_{\text{Lip}}$. We assume $\|\hat{\theta}_{AE}\|_{EE} \leq \|\theta^*\|_{EE}$ and $\frac{1}{N}\sum_{i = 1}^N \sum_{t=1}^T q_tL(\theta^*, Z_{t}^{(i)})\leq\xi$. For any $\delta > 0$, with probability at least $1-\delta$, the following inequality holds:
{\setlength{\mathindent}{0cm}
\begin{equation}
\begin{aligned}
 &\mathbb{E} \mathcal{L}(\hat{\theta}_{AE}, Z_{T+1}) - \mathbb{E} \mathcal{L}(\theta^*, Z_{T+1})\\ 
 &\leq 2\text{disc}_T(\bm{q}) + 6M\sqrt{4\pi logT}\mathcal{R}_T^{sq}(L\circ\Theta_{AE}) + \sigma \\
& + \|L\|_{\text{Lip}}\cdot\mathbb{E}_G \mathcal{W}(Z_{T+1}, g Z_{T+1}) - \|\bm{q}\|_1\|\hat{\theta}_{AE}\|_{EE} + 2\xi \nonumber
\end{aligned}
\end{equation}
}
where $\sigma =  \sqrt{\frac{2 \log (2/\sigma)}{N}} +\|\bm{q}\|_2(M\sqrt{8\text{log}\frac{1}{\delta}} + 1)$. 
\end{corollary}

To put it simply, when the data do not have perfect symmetries, approximately equivariant models may have better prediction performance than data augmentations and perfectly equivariant models because of the term $-\|\bm{q}\|_1\|\hat{\theta}_{AE}\|_{EE}$ in the bound. The empirical error of the population minimizer $\xi$ can be small enough to be ignored. 

If approximately equivariant estimators can learn the correct amount of symmetry in the data, which means that $\|\hat{\theta}_{AE}\|_{EE}$ is big and close to the true equivariance error in the data $\|\theta^*\|_{EE}$, then they tend to have better generalizability. On the contrary, the estimators trained on a uniformly augmented training set and perfectly equivariant estimators maintain zero equivariance error even when data are not perfectly symmetric, which is overly restricted.

%% file: sections/conclusion.tex
We take the first steps in the theoretical understanding of data augmentation and equivariant networks on the task of non-stationary dynamics forecasting. We derive the generalization bounds and show that strictly equivariant networks have a tighter upper bound than data augmentation, and that approximately equivariant estimators can outperform both data augmentation and perfectly equivariant networks on modeling imperfectly symmetric dynamics. A limitation of this work is that our theoretical comparison is only for upper bounds, which can be arbitrarily loose in practice. Future work includes improving the generalizing bounds with Pac-Bayesian analysis and  deriving  lower bounds for these approaches characterizing the hardness of learning for different model classes. 

%% file: sections/app.tex
\begin{theorem}\label{proof:bounds-dyanmics}
For any $\delta > 0$, with probability at least $1-\delta$, the following inequality holds for all $\hat{\theta} \in \Theta$ and all $\alpha = \|\bm{q}\|_2/2>0$:
$$\mathbb{E} L(\hat{\theta}, Z_{T+1}) - \mathbb{E} L(\theta^*, Z_{T+1}) \leq 2\text{disc}_T(\bm{q}) + \sqrt{\frac{2 \log (2/\sigma)}{N}} +6M\sqrt{4\pi logT}\mathcal{R}_T^{sq}(L\circ\Theta)  + \|\bm{q}\|_2(M\sqrt{8\text{log}\frac{1}{\delta}} + 1)$$
\end{theorem}
\vskip 0.2in

\begin{proof}
$$ \mathbb{E} L(\hat{\theta}, Z_{T+1}) - \mathbb{E} L(\theta^*, Z_{T+1}) = I + II + III + IV$$
$$I = \mathbb{E} L(\hat{\theta}, Z_{T+1}) - \frac{1}{N}\sum_{i = 1}^N \sum_{t=1}^T q_tL(\hat{\theta}, Z_{t}^{(i)})$$
% $\mathbb{E} L(\hat{\theta}, Z_{T+1}) - \frac{1}{N}\sum_{i = 1}^N \sum_{t=1}^T q_tL(\hat{\theta}, Z_{t}^{(i)}) \leq \alpha_1 \mathcal{R}_T^{\text{seq}}(\mathcal{H}) + \alpha_2 \text{disc}_T(\bm{q}) + \alpha_3 \mathbb{E}_G \mathcal{W}(Z_{T+1}, g Z_{T+1}) +  C$

$$II = \frac{1}{N}\sum_{i = 1}^N \sum_{t=1}^T q_t L(\hat{\theta}, Z_{t}^{(i)}) - \frac{1}{N}\sum_{i = 1}^N \sum_{t=1}^T q_tL(\theta^*, Z_{t}^{(i)}) \leq 0 \text{(the model does not underfit the data)}$$ 

$$III = \frac{1}{N}\sum_{i = 1}^N \sum_{t=1}^T q_t L(\theta^*, Z_{t}^{(i)}) - \mathbb{E}\sum_{t=1}^T q_tL(\theta^*, Z_t) \leq \sqrt{\frac{2 \log (2/\sigma)}{N}} \text{(time series are i.i.d sampled).} $$

$$IV = \mathbb{E}\sum_{t=1}^T q_tL(\theta^*, Z_t) - \mathbb{E}L(\theta^*, Z_{T+1}) \leq \text{sup}_{\theta \in \Theta}|\mathbb{E}\sum_{t=1}^T q_tL(\theta, Z_t) - \mathbb{E}L(\theta, Z_{T+1})| = \text{disc}_T(\bm{q})$$

Now we only need to bound the first term $I = \mathbb{E} L(\hat{\theta}, Z_{T+1}) - \frac{1}{N}\sum_{i = 1}^N \sum_{t=1}^T q_tL(\hat{\theta}, Z_{t}^{(i)})$
\begin{align*}
    & \mathbb{P}(I - \text{disc}_T(\bm{q}) > \epsilon)\\
    & = \mathbb{P}(\mathbb{E} L(\hat{\theta}, Z_{T+1}) - \frac{1}{N}\sum_{i = 1}^N \sum_{t=1}^T q_tL(\hat{\theta}, Z_{t}^{(i)})-\text{sup}_{\theta \in \Theta}|\mathbb{E}L(\theta, Z_{T+1}) - \mathbb{E}\sum_{t=1}^T q_tL(\theta, Z_t)|> \epsilon)\\
    & \leq \mathbb{P}(\text{sup}_{\theta \in \Theta}|\mathbb{E} L(\theta, Z_{T+1}) - \frac{1}{N}\sum_{i = 1}^N \sum_{t=1}^T q_tL(\theta, Z_{t}^{(i)})| -\text{sup}_{\theta \in \Theta}|\mathbb{E}L(\theta, Z_{T+1}) - \mathbb{E}\sum_{t=1}^T q_tL(\theta, Z_t)|> \epsilon)\\
    & \leq \mathbb{P}(\text{sup}_{\theta \in \Theta}|\mathbb{E} L(\theta, Z_{T+1}) - \frac{1}{N}\sum_{i = 1}^N \sum_{t=1}^T q_tL(\theta, Z_{t}^{(i)}) + \mathbb{E}\sum_{t=1}^T q_tL(\theta, Z_t) - \mathbb{E}L(\theta, Z_{T+1})|> \epsilon)\\
     & \leq \mathbb{P}(\text{sup}_{\theta \in \Theta}|\mathbb{E} L(\theta, Z_{T+1}) - \frac{1}{N}\sum_{i = 1}^N \sum_{t=1}^T q_tL(\theta, Z_{t}^{(i)}) + \mathbb{E}\sum_{t=1}^T q_tL(\theta, Z_t) - \mathbb{E}L(\theta, Z_{T+1})|> \epsilon)\\
     & = \mathbb{P}(\text{sup}_{\theta \in \Theta}|\mathbb{E}\sum_{t=1}^T q_tL(\theta, Z_t)- \frac{1}{N}\sum_{i = 1}^N \sum_{t=1}^T q_tL(\theta, Z_{t}^{(i)}) |> \epsilon)\\
     & = \mathbb{P}(\text{exp}(\lambda\; \text{sup}_{\theta \in \Theta}|\mathbb{E}\sum_{t=1}^T q_tL(\theta, Z_t)- \frac{1}{N}\sum_{i = 1}^N \sum_{t=1}^T q_tL(\theta, Z_{t}^{(i)})) |))> \text{exp}(\lambda\epsilon)) \\
     & \leq \text{exp}(-\lambda\epsilon)\mathbb{E}[\text{exp}(\lambda\; \text{sup}_{\theta \in \Theta}(\mathbb{E}\sum_{t=1}^T q_tL(\theta, Z_t)- \frac{1}{N}\sum_{i = 1}^N \sum_{t=1}^T q_tL(\theta, Z_{t}^{(i)})))] \;\; \text{(by Markov’s inequality)}
\end{align*}

\begin{align*}
    & \mathbb{E}[\text{exp}(\lambda\; \text{sup}_{\theta \in \Theta}(\mathbb{E}\sum_{t=1}^T q_tL(\theta, Z_t)- \frac{1}{N}\sum_{i = 1}^N \sum_{t=1}^T q_tL(\theta, Z_{t}^{(i)})))] \\
    & = \mathbb{E}[\text{exp}(\lambda\;\text{sup}_{\theta \in \Theta}(\frac{1}{N}\sum_{i = 1}^N \sum_{t=1}^Tq_t(\mathbb{E} L(\theta, Z_t)-  L(\theta, Z_{t}^{(i)})))] \\
    & = \mathbb{E}[\text{exp}(\lambda\;\text{sup}_{\theta \in \Theta}(\frac{1}{N}\sum_{i = 1}^N \sum_{t=1}^Tq_t(\mathbb{E}[L(\theta, Z_t) |  Z_0]-  L(\theta, Z_{t}^{(i)})))] \\
    & \leq \mathbb{E}[\mathbb{E}_{Z_0\sim\mathcal{X}^k\times\mathcal{X}}\text{exp}(\lambda\;\text{sup}_{\theta \in \Theta}[(\frac{1}{N}\sum_{i = 1}^N \sum_{t=1}^Tq_t(L(\theta, Z_t) -  L(\theta, Z_{t}^{(i)}))| Z_0])] \\
    & = \mathbb{E}[\text{exp}(\lambda\;\text{sup}_{\theta \in \Theta}[(\frac{1}{N}\sum_{i = 1}^N \sum_{t=1}^Tq_t(L(\theta, Z'_t) -  L(\theta, Z_{t}^{(i)})))] \\
    & = \mathbb{E}[\text{exp}(\lambda\;\text{sup}_{\theta \in \Theta}[( \sum_{t=1}^Tq_t(L(\theta, Z'_t) -  L(\theta, Z^*_{t})))] \\
    & = \mathbb{E}\mathbb{E}_{\bm{\sigma}}[\text{exp}(\lambda\;\text{sup}_{\theta \in \Theta}[( \sum_{t=1}^T\sigma_t q_t(L(\theta, Z'_t) -  L(\theta, Z^*_{t})))] \\
    & = \mathbb{E}_{\bm{z}^*, \bm{z}'}\mathbb{E}_{\bm{\sigma}}[\text{exp}(\lambda\;\text{sup}_{\theta \in \Theta}[( \sum_{t=1}^T\sigma_t q_t(f(\bm{z}^*_t(\bm{\sigma})) -  f(\bm{z}'_t(\bm{\sigma}))))]\; \text{replace} \; L(\theta, \cdot)\;  \text{with} \;  f \;  \text{for simplicity}. \\
    & \leq \mathbb{E}_{\bm{z}^*}\mathbb{E}_{\bm{\sigma}}[\text{exp}(2\lambda\;\text{sup}_{f \in \mathcal{F}}\sum_{t=1}^T\sigma_t q_t f(\bm{z}^*_t(\bm{\sigma})))]
\end{align*}

Given $\bm{z}^*$, let $C$ be the minimal $\alpha$-cover of $\mathcal{F}$ on $\bm{z}^*$,
$$\text{sup}_{f \in \mathcal{F}}\sum_{t=1}^T\sigma_t q_t f(\bm{z}^*_t(\bm{\sigma}))) \leq \text{max}_{\bm{c}\in C}\sum_{t=1}^T\sigma_t q_t \bm{c}_t(\bm{\sigma}) + \alpha$$
Thus, 
\begin{align*}
    & \mathbb{E}_{\bm{\sigma}}[\text{exp}(2\lambda\;\text{sup}_{f \in \mathcal{F}}\sum_{t=1}^T\sigma_t q_t f(\bm{z}^*_t(\bm{\sigma})))]\\
    & \leq \text{exp}(2\lambda\alpha)\mathbb{E}_{\bm{\sigma}}[\text{exp}(2\lambda\;\text{max}_{\bm{c}\in C}\sum_{t=1}^T\sigma_t q_t  \bm{c}_t(\bm{\sigma}))] \\
    & \leq  \text{exp}(2\lambda\alpha)\text{max}_{\bm{c}\in C}\mathbb{E}_{\bm{\sigma}}[\text{exp}(2\lambda\;\sum_{t=1}^T\sigma_t q_t \bm{c}_t(\bm{\sigma}))] \\
    & =  \text{exp}(2\lambda\alpha)\text{max}_{\bm{c}\in C}\mathbb{E}_{\bm{\sigma}}[\text{exp}(2\lambda\;\sum_{t=1}^{T-1}\sigma_t q_t \bm{c}_t(\bm{\sigma}))\; \mathbb{E}_{\bm{\sigma}^T} [\text{exp}(2\lambda \; \sigma_T q_T \bm{c}_T(\bm{\sigma}))|\bm{\sigma}_{1:T-1}]] \\
    & \leq \text{exp}(2\lambda\alpha)\text{max}_{\bm{c}\in C}\mathbb{E}_{\bm{\sigma}}[\text{exp}(2\lambda\;\sum_{t=1}^{T-1}\sigma_t q_t \bm{c}_t(\bm{\sigma}))\; \text{exp}(2\lambda^2 q_T^2 M^2) ]\\
    & \leq \text{exp}(2\lambda\alpha) \text{exp}(2\lambda^2 \|\bm{q}\|^2_2 M^2) \; (\text{Iterate the last inequality over} \; t)
\end{align*}
Then we have
\begin{align*}
&\mathbb{P}(I - \text{disc}_T(\bm{q}) > \epsilon) \leq \mathbb{E}_{\bm{z}}[\mathcal{N}_1(\alpha, \Theta, \bm{z})] \text{exp}(2\lambda\alpha - \lambda\epsilon + 2\lambda^2 \|\bm{q}\|^2_2 M^2)
\end{align*}
Optimize $\lambda$
\begin{align*}
&\mathbb{P}(I - \text{disc}_T(\bm{q}) > \epsilon) \leq \mathbb{E}_{\bm{z}}[\mathcal{N}_1(\alpha, \Theta, \bm{z})] \text{exp}(\frac{(\epsilon - 2\alpha)^2}{8\|\bm{q}\|^2_2 M^2})
\end{align*}
Finally, $\mathbb{E}_{\bm{z}}[\mathcal{N}_1(\alpha, \Theta, \bm{z})]$ can be further bounded by the sequential Rademacher complexity based on the Theorem 2 in \cite{kuznetsov2020discrepancy}. 
\end{proof}

\begin{corollary}\label{proof:bounds-dyanmics-sym}
Let $L(\theta, \cdot)$ be Lipschitz uniformly over $\theta$, with a Lipschitz constant $\|L\|_{\text{Lip}}$. Assume $\frac{1}{N}\sum_{i = 1}^N \sum_{t=1}^T q_tL(\theta^*, Z_{t}^{(i)})\leq\xi$. For any $\delta > 0$, with probability at least $1-\delta$, the following inequality holds for all $\alpha = \|\bm{q}\|_2/2>0$:

{\setlength{\mathindent}{0cm}
\begin{equation}
\begin{aligned}
\mathbb{E} L(\hat{\theta}_G, Z_{T+1}) - \mathbb{E} L(\theta^*, Z_{T+1}) & \leq 2\text{disc}_T(\bm{q}) + 6M\sqrt{4\pi logT}\mathcal{R}_T^{sq}(L\circ\Theta_G) + \Delta \\
& + \textrm{max}_{t,i}\;\|L\|_{\text{Lip}}\cdot\mathbb{E}_G [\mathcal{W}(Z_{T+1}, g Z_{T+1}) + q_t\mathcal{W}(Z_{t}^{(i)}, g Z_{t}^{(i)})]
\end{aligned}
\end{equation}
}

{\setlength{\mathindent}{0cm}
\begin{equation}
\begin{aligned}
\mathbb{E} L(\hat{\theta}_E, Z_{T+1}) - \mathbb{E} L(\theta^*, Z_{T+1}) & \leq 2\text{disc}_T(\bm{q}) +6M\sqrt{4\pi logT}\mathcal{R}_T^{sq}(L\circ\Theta_E) + \Delta \\
& + \textrm{max}_{t,i}\;\|L\|_{\text{Lip}}\cdot\mathbb{E}_G \mathcal{W}(Z_{t}^{(i)}, g Z_{t}^{(i)})
\end{aligned}
\end{equation}
}

{\setlength{\mathindent}{0cm}
\begin{equation}
\begin{aligned}
\mathbb{E} L(\hat{\theta}_{AE}, Z_{T+1}) - \mathbb{E} L(\theta^*, Z_{T+1}) & \leq 2\text{disc}_T(\bm{q}) + 6M\sqrt{4\pi logT}\mathcal{R}_T^{sq}(L\circ\Theta_{AE}) + \Delta \\
& + \|L\|_{\text{Lip}}\cdot\mathbb{E}_G \mathcal{W}(Z_{T+1}, g Z_{T+1}) - \|\bm{q}\|_1\|\hat{\theta}_{AE}\|_{EE} + 2\xi
\end{aligned}
\end{equation}
}
where $\Delta = \sqrt{\frac{2 \log (2/\sigma)}{N}} + \|\bm{q}\|_2(M\sqrt{8\text{log}\frac{1}{\delta}} + 1)$. 
\end{corollary}

\begin{proof}

When $\hat{\theta} = \hat{\theta}_G$, we only need to derive a bound for I in the previous proof. 
{\setlength{\mathindent}{0cm}
\begin{equation}
\begin{aligned}
&I = \mathbb{E} L(\hat{\theta}, Z_{T+1}) - \frac{1}{N}\sum_{i = 1}^N \sum_{t=1}^T q_t L(\hat{\theta}, Z_{t}^{(i)}) = A + B + C \\
&A = \mathbb{E} L(\hat{\theta}, Z_{T+1}) - \mathbb{E}\mathbb{E}_G L(\hat{\theta}, g Z_{T+1}) \leq \mathbb{E}_G |\mathbb{E} L(\hat{\theta}_G, Z_{T+1}) - \mathbb{E}L(\hat{\theta}_G, g Z_{T+1})|  \leq \|L\|_{\text{Lip}}\cdot\mathbb{E}_G \mathcal{W}(Z_{T+1}, g Z_{T+1})\\
&B = \mathbb{E}\mathbb{E}_G L(\hat{\theta}, g Z_{T+1}) - \frac{1}{N}\sum_{i = 1}^N \sum_{t=1}^T q_t\mathbb{E}_GL(\hat{\theta}, g Z_{t}^{(i)})  \leq \text{disc}_T(\bm{q}) +6M\sqrt{4\pi logT}\mathcal{R}_T^{sq}(L\circ\Theta_G) + \Delta \\
&C = \frac{1}{N}\sum_{i = 1}^N \sum_{t=1}^T q_t\mathbb{E}_GL(\hat{\theta}, g Z_{t}^{(i)}) - \frac{1}{N}\sum_{i = 1}^N \sum_{t=1}^T q_tL(\hat{\theta}, Z_{t}^{(i)}) \leq \frac{1}{N}\sum_{i = 1}^N \sum_{t=1}^T q_t[\mathbb{E}_GL(\hat{\theta}_G, g Z_{t}^{(i)}) - L(\hat{\theta}_G, Z_{t}^{(i)})] \\ 
& \;\;\;\leq \textrm{max}_{t,i}\;\|L\|_{\text{Lip}}\cdot\mathbb{E}_G \mathcal{W}(Z_{t}^{(i)}, g Z_{t}^{(i)})
\end{aligned}
\end{equation}
}

When $\hat{\theta} = \hat{\theta}_{AE}$: 
$$ \mathbb{E} L(\hat{\theta}_{AE}, Z_{T+1}) - \mathbb{E} L(\theta^*, Z_{T+1}) = I + II + III + IV + V + VI$$
{\setlength{\mathindent}{0cm}
\begin{equation}
\begin{aligned}
&I = \mathbb{E} L(\hat{\theta}_{AE}, Z_{T+1}) - \frac{1}{N}\sum_{i = 1}^N \sum_{t=1}^T q_tL(\hat{\theta}_{AE}, Z_{t}^{(i)}) \leq \text{disc}_T(\bm{q}) +6M\sqrt{4\pi logT}\mathcal{R}_T^{sq}(L\circ\Theta_{AE}) + \Delta \\
&II = \frac{1}{N}\sum_{i = 1}^N \sum_{t=1}^T q_tL(\hat{\theta}_{AE}, Z_{t}^{(i)}) - \frac{1}{N}\sum_{i = 1}^N \sum_{t=1}^T q_tL(\theta^*, Z_{t}^{(i)}) \leq 0 \;\text{(The model does not underfit the data)}\\
&IV = \frac{1}{N}\sum_{i = 1}^N \sum_{t=1}^T q_t \mathbb{E}_GL(\theta^*, g Z_{t}^{(i)}) - \sum_{t=1}^T q_t \mathbb{E}\mathbb{E}_GL(\theta^*, g Z_{t})\leq \sqrt{\frac{2 \log (2/\sigma)}{N}} \\
&V = \sum_{t=1}^T q_t \mathbb{E}\mathbb{E}_GL(\theta^*, g Z_{t}^{(i)}) - \mathbb{E}\mathbb{E}_GL(\theta^*, g Z_{T+1}) = \sum_{t=1}^T q_t \mathbb{E}\bar{L}(\theta^*, g Z_{t}^{(i)}) - \mathbb{E}\bar{L}(\theta^*, g Z_{T+1})  \\
& \;\;\;\leq \text{sup}_{\theta \in \Theta}\left|\mathbb{E}\left[\sum_{t=1}^T q_t \bar{L}(\theta, Z_t)\right] - \mathbb{E}\bar{L}(\theta, Z_{T+1})\right| \leq \text{disc}_T(\bm{q})\\
&VI = \mathbb{E}\mathbb{E}_GL(\theta^*, g Z_{T+1}) - \mathbb{E}L(\theta^*, Z_{T+1}) \leq \|L\|_{\text{Lip}}\cdot\mathbb{E}_G \mathcal{W}(Z_{T+1}, g Z_{T+1}) \\
&III = \frac{1}{N}\sum_{i = 1}^N \sum_{t=1}^T q_tL(\theta^*, Z_{t}^{(i)}) - \frac{1}{N}\sum_{i = 1}^N \sum_{t=1}^T q_t \mathbb{E}_GL(\theta^*, g Z_{t}^{(i)}) \;\;\text{(let $x_t^{(i)} = X_{t-k-1:t-1}^{(i)}$ and $y_t^{(i)} = X_{t}^{(i)}$)}\\
& \;\;\; = \frac{1}{N}\sum_{i = 1}^N \sum_{t=1}^T q_t \mathbb{E}_G [\|y_t^{(i)} - f_{\theta^*}(x_t^{(i)}) \| - \|g y_t^{(i)} - f_{\theta^*}(gx_t^{(i)}) \|]\\
& \;\;\; = \frac{1}{N}\sum_{i = 1}^N \sum_{t=1}^T q_t \mathbb{E}_G [\|y_t^{(i)} - f_{\theta^*}(x_t^{(i)}) \| - \|g y_t^{(i)} - gf_{\theta^*}(x_t^{(i)}) + gf_{\theta^*}(x_t^{(i)})- f_{\theta^*}(gx_t^{(i)}) \|]\\
& \;\;\; \leq \frac{1}{N}\sum_{i = 1}^N \sum_{t=1}^T q_t \mathbb{E}_G [-\|gf_{\theta^*}(x_t^{(i)})- f_{\theta^*}(gx_t^{(i)})\| + \|y_t^{(i)} - f_{\theta^*}(x_t^{(i)}) \| + \|g y_t^{(i)} - gf_{\theta^*}(x_t^{(i)})  \|]\\
& \;\;\; \leq -\|\bm{q}\|_1\|\hat{\theta}_{AE}\|_{EE} + 2\xi
\end{aligned}
\end{equation}
}

When $\hat{\theta} = \hat{\theta}_E$: 
$$ \mathbb{E} L(\hat{\theta}_{E}, Z_{T+1}) - \mathbb{E} L(\theta^*, Z_{T+1}) = I + II + III + IV + V + VI$$
{\setlength{\mathindent}{0cm}
\begin{equation}
\begin{aligned}
&I = \mathbb{E} L(\hat{\theta}_{E}, Z_{T+1}) - \frac{1}{N}\sum_{i = 1}^N \sum_{t=1}^T q_tL(\hat{\theta}_{E}, Z_{t}^{(i)}) \leq \text{disc}_T(\bm{q}) +6M\sqrt{4\pi logT}\mathcal{R}_T^{sq}(L\circ\Theta_{E}) + \Delta \\
&II = \frac{1}{N}\sum_{i = 1}^N \sum_{t=1}^T q_tL(\hat{\theta}_{E}, Z_{t}^{(i)}) - \frac{1}{N}\sum_{i = 1}^N \sum_{t=1}^T q_tL(\theta^*, Z_{t}^{(i)}) \leq 0 \;\text{(The model does not underfit the data)}\\
&III = \frac{1}{N}\sum_{i = 1}^N \sum_{t=1}^T q_tL(\theta^*, Z_{t}^{(i)}) - \frac{1}{N}\sum_{i = 1}^N \sum_{t=1}^T q_t \mathbb{E}_GL(\theta^*, g Z_{t}^{(i)}) = 0\\
&IV = \frac{1}{N}\sum_{i = 1}^N \sum_{t=1}^T q_t \mathbb{E}_GL(\theta^*, g Z_{t}^{(i)}) - \sum_{t=1}^T q_t \mathbb{E}\mathbb{E}_GL(\theta^*, g Z_{t})\leq \sqrt{\frac{2 \log (2/\sigma)}{N}} \\
&V = \sum_{t=1}^T q_t \mathbb{E}\mathbb{E}_GL(\theta^*, g Z_{t}^{(i)}) - \mathbb{E}\mathbb{E}_GL(\theta^*, g Z_{T+1}) = \sum_{t=1}^T q_t \mathbb{E}\bar{L}(\theta^*, g Z_{t}^{(i)}) - \mathbb{E}\bar{L}(\theta^*, g Z_{T+1})  \\
& \;\;\;\leq \text{sup}_{\theta \in \Theta}\left|\mathbb{E}\left[\sum_{t=1}^T q_t \bar{L}(\theta, Z_t)\right] - \mathbb{E}\bar{L}(\theta, Z_{T+1})\right| \leq \text{disc}_T(\bm{q})\\
&VI = \mathbb{E}\mathbb{E}_GL(\theta^*, g Z_{T+1}) - \mathbb{E}L(\theta^*, Z_{T+1}) \leq \|L\|_{\text{Lip}}\cdot\mathbb{E}_G \mathcal{W}(Z_{T+1}, g Z_{T+1}) 
\end{aligned}
\end{equation}
}
Combining the bounds for the six terms gives the desired result.

Moreover, 
{\setlength{\mathindent}{0cm}
\begin{equation}\label{app:complex-sym-aug}
\begin{aligned}
& \mathcal{R}_T^{sq}(L\circ\Theta_E) - \mathcal{R}_T^{sq}(\bar{L}\circ\Theta) \\
& \leq  \mathbb{E}_{\bm{\sigma}}[\text{sup}_{\theta_E\in\Theta_E}\frac{1}{N} \sum_{i=1}^N\sum_{t=1}^T\sigma_t q_t L(\theta_E, Z_t)] - \mathbb{E}_{\bm{\sigma}}[\text{sup}_{\theta\in\Theta}\frac{1}{N} \sum_{i=1}^N\sum_{t=1}^T\sigma_t q_t \mathbb{E}_GL(\theta, gZ_t)] \\
& \leq \mathbb{E}_{\bm{\sigma}}[\text{sup}_{\theta_E\in\Theta_E}\frac{1}{N} \sum_{i=1}^N\sum_{t=1}^T\sigma_t q_t \mathbb{E}_GL(\theta_E, gZ_t)] - \mathbb{E}_{\bm{\sigma}}[\text{sup}_{\theta\in\Theta}\frac{1}{N} \sum_{i=1}^N\sum_{t=1}^T\sigma_t q_t \mathbb{E}_GL(\theta, gZ_t)] \\
& = \mathbb{E}_{\bm{\sigma}}\mathbb{E}_G[\text{sup}_{\theta_E\in\Theta_E}\frac{1}{N} \sum_{i=1}^N\sum_{t=1}^T\sigma_t q_t L(\theta_E, gZ_t) - \text{sup}_{\theta\in\Theta}\frac{1}{N} \sum_{i=1}^N\sum_{t=1}^T\sigma_t q_t L(\theta, gZ_t)]\\
& \leq \mathbb{E}_{\bm{\sigma}}\mathbb{E}_G[\text{sup}_{\theta_E\in\Theta_E}\frac{1}{N} \sum_{i=1}^N\sum_{t=1}^T\sigma_t q_t L(\theta_E, gZ_t) - \text{sup}_{\theta\in\Theta_E}\frac{1}{N} \sum_{i=1}^N\sum_{t=1}^T\sigma_t q_t L(\theta, gZ_t)]\\
& = 0
\end{aligned}
\end{equation}
}
\end{proof}